\documentclass{article}
\usepackage{iclr2027_conference,times}

\usepackage{amsmath,amsfonts,bm}

\def\eqref#1{equation~\ref{#1}}
\def\1{\bm{1}}

\DeclareMathAlphabet{\mathsfit}{\encodingdefault}{\sfdefault}{m}{sl}
\SetMathAlphabet{\mathsfit}{bold}{\encodingdefault}{\sfdefault}{bx}{n}

\usepackage{xcolor}
\definecolor{lightblue}{RGB}{32,89,232}
\usepackage{graphicx}
\usepackage[backref, colorlinks=true, citecolor=lightblue]{hyperref}
\usepackage{url}
\usepackage{amsmath}
\usepackage{amssymb}
\usepackage{mathtools}
\usepackage{amsthm}
\usepackage{enumitem}
\usepackage{thm-restate}
\usepackage[ruled,vlined,noend,linesnumbered]{algorithm2e}
\usepackage[noend]{algorithmic}
\usepackage{subfigure}
\usepackage{multirow,makecell}
\usepackage{booktabs}
\usepackage{float}     % [H] placement specifier
\usepackage{placeins}  % \FloatBarrier
\usepackage{tikz}
\usepackage{pgfplots}
\pgfplotsset{compat=1.18}
\usepackage[capitalize,noabbrev]{cleveref}
\usepackage{fontawesome5}

\newcommand{\ourloop}{\textsc{LoopSpec}}

\newtheorem{theorem}{Theorem}

\newtheorem{lemma}[theorem]{Lemma}

\newtheorem{definition}[theorem]{Definition}

\newcommand{\abs}[1]{\left |#1\right|}
\newcommand{\tv}[1]{\mathrm{TV}\left({#1}\right)}

\newcommand{\ourosmall}{$\mathtt{Ouro\text{-}1.4B}$\xspace}
\newcommand{\ourolarge}{$\mathtt{Ouro\text{-}2.6B}$\xspace}

\newcommand{\ourosmallthinking}{$\mathtt{Ouro}\text{-}\mathtt{1.4B}\text{-}\mathtt{Thinking}$}
\newcommand{\ourolargethinking}{$\mathtt{Ouro}\text{-}\mathtt{2.6B}\text{-}\mathtt{Thinking}$}

\newcommand{\ravenllama}{$\mathtt{Raven\text{-}Llama\text{-}3.2}$\xspace}
\newcommand{\ravenolmo}{$\mathtt{Raven\text{-}OLMo\text{-}2\text{-}0425}$\xspace}

\title{\ourloop{}: Pipelined Self-Speculative Decoding for Looped Transformers}

\iclrfinalcopy
\author{
SangLyul Cho$^{1}$\thanks{Equal contribution.} \quad
Langqing Cui$^{2}$\footnotemark[1] \quad
Sehoon Kim$^{2}$ \quad
Dongsu Han$^{2}$ \quad
Insu Han$^{2}$\thanks{Corresponding author.}  \\
$^{1}$Seoul National University \qquad
$^{2}$KAIST
}

\begin{document}

\maketitle
\lhead{Preprint.}   % 반드시 \maketitle '뒤에' 와야 함

\begin{abstract}
Looped Transformers achieve strong performance with compact parameter sizes by repeatedly applying a shared stack of Transformer blocks across recurrent depths.
However, they incur higher decoding latency than standard Transformer models of comparable parameter size because shared weights are accessed at every recurrent depth.
To improve decoding efficiency, self-speculative decoding is particularly well suited to Looped Transformers, as their intermediate recurrent states can directly provide draft predictions without an auxiliary draft model.
We therefore propose \ourloop{}, a training-free self-speculative decoding framework tailored for Looped Transformers.
\ourloop{} extracts draft tokens from early recurrent states and operates in a pipelined manner, overlapping draft generation of future tokens with target verification of the current token.
To improve draft accuracy without excessive compute overhead, we introduce a selective second proposal from deeper recurrent depth while ensuring lossless decoding under both greedy and sampling regimes.
Furthermore, we derive the optimal proposal depths in closed form and show the prediction matches measurement.
Across reasoning and coding benchmarks, \ourloop{} achieves up to 6.83$\times$ inference speedup across diverse Looped Transformers.
\end{abstract}

\begin{center}
\faGlobe~\href{https://langq1225.github.io/loopspec/}{$\mathtt{Project\ Page:}$}
{\small\url{https://langq1225.github.io/loopspec/}} \\[3pt]
\faGithub~\href{https://github.com/kaist-flexml-lab/loopspec}{$\mathtt{Code:}$} {\small\url{https://github.com/kaist-flexml-lab/loopspec}}
\end{center}

\section{Introduction}

Looped Transformers \citep{geiping2025scaling,zhu2025scaling,mcleish2026teaching,nanbeige2026nanbeige} have recently emerged as an efficient architectural paradigm that iteratively applies shared Transformer blocks across recurrent depths, drastically reducing the total parameter count. 
For instance, a model of four blocks applied 8 times achieves the effective depth of a 32-layer model while storing only four layers of parameters.
By expanding computational depth through this recurrence, Looped Transformers match or outperform standard Transformers of comparable parameter count on reasoning tasks.
However, executing shared Transformer blocks across multiple recurrent depths repeatedly fetches identical parameters within each token generation step, leaving inference memory-bandwidth bound thus increasing decoding latency \citep{hooper2023speed}.

While speculative decoding with a separately trained draft model \citep{leviathan2023fast,chen2023accelerating,li2025eagle3,chen2026dflash} has emerged as a promising approach for lossless inference acceleration, it introduces additional training costs, potential distribution mismatch between draft and target models, and extra memory.
Self-speculative decoding \citep{zhang2024draft,elhoushi2024layerskip,cha2026knapspec} avoids these limitations by using intermediate computations of the target model itself to generate draft tokens, eliminating the need for a separate draft model.
Looped Transformers are naturally suited to this paradigm, since the intermediate representation in each recurrent depth forms a sequence of increasingly refined proposals toward the final prediction.

\begin{figure}[t]
\centering
\includegraphics[width=1\linewidth]{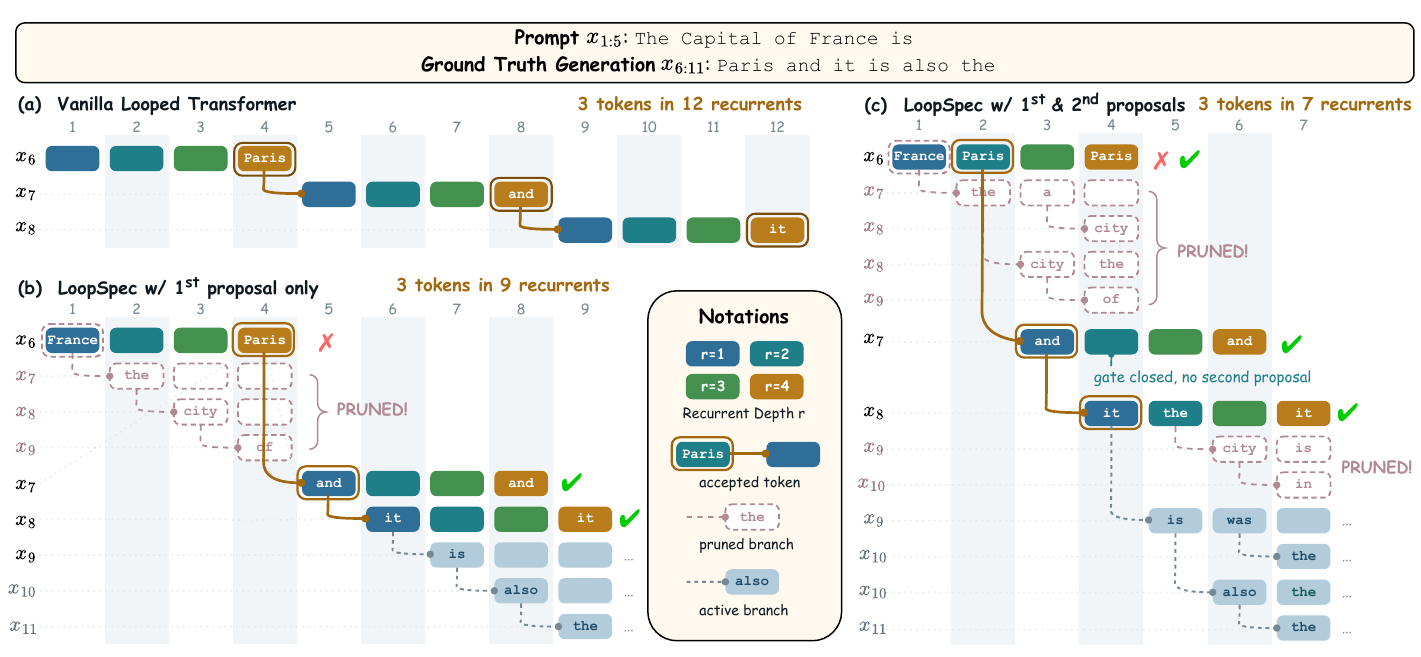}
\vspace{-1.5em}
\caption{
\ourloop{} pipelines decoding over $R=4$ recurrent depth.
(a)~A vanilla Looped Transformer finishes all $R$ recurrences of one token before starting the next.
(b)~\ourloop{} with a first proposal only at $d_1=1$: an early-depth draft starts the next token immediately, but a rejection stalls the pipeline (\cref{sec:method-pipelining}).
(c)~\ourloop{} with first proposal at $d_1=1$ and a residual second proposal with gating mechanism at $d_2=2$: the second proposal creates a fallback branch, causing fewer stalls in the pipeline (\cref{sec:method-second-proposal}).
%\sehoon{can we put forward references to (b) and (c) to the associated subsections?}
%\sehoon{the term `gating' is too out-of-the-blue here. We should at least briefly explain what it is. e.g. LOOPSPEC with first proposal at d1 = 1 and a residual second proposal at d2 = 2. Gating mechanism in x7 is triggered when xxx to avoid yyy.}
% \sehoon{mention that the proposal depth d1 = 1 for (b) and d1=1, d2=2 for (c)}
}
\label{fig:pipeline}
\end{figure}

To this end, we propose {\ourloop{}}, a training-free self-speculative pipelining framework specifically designed for Looped Transformers. 
The key idea is to exploit the token predictions from the intermediate recurrent depths. 
Rather than waiting for all recurrent steps to finish, \ourloop{} drafts a token from an early recurrent state and immediately starts computing the subsequent token conditioned on this draft (\cref{fig:pipeline}(b)).
The original computation and this speculative continuation use the same Transformer blocks, so they can be processed together in one batch even at different recurrent depths. 
Once the original computation reaches its final depth, it verifies the early draft.
If the draft is accepted, some recurrent steps for the subsequent token have already been completed, so fewer steps are needed before it can reach the final depth.
Otherwise, we prune the speculative draft and restart from the verified prefix.

Since deeper recurrent computations progressively refine the model's prediction, \ourloop{} uses a deeper recurrent state to produce a second proposal and starts an alternative continuation before verification.
Accepting this second proposal preserves its progress when the first proposal is rejected (\cref{fig:pipeline}(c)).
The resulting computation forms a pipeline in which the target computation and multiple speculative continuations
advance in parallel, substantially increasing hardware utilization while preserving the exact target-model distribution. \cref{fig:pipeline} provides an overview of the complete \ourloop{} pipeline. 

Our primary contributions can be summarized as follows:
\begin{itemize}[leftmargin=1.5em, labelsep=0.8em, itemsep=0.1em]
    \item We propose \ourloop{}, a training-free self-speculative pipelining framework that accelerates Looped Transformers without external drafters or structural modifications. To improve draft acceptance while limiting computational overhead, \ourloop{} introduces
    (i) a \textit{residual} second proposal from a deeper recurrent state to recover from first proposal rejection, and
    (ii) a \textit{gating mechanism} that creates the fallback speculative continuation only when the deeper state disagrees with the first draft (\cref{sec:method-pipelining,sec:method-second-proposal}).
    \item 
    We prove that the restart probability depends only on the second proposal depth, not the first.
    Using the empirically observed power-law decay of the intermediate-to-target TV distance, we derive the optimal proposal depths in closed form matching the measured optima (\cref{sec:theory}).
    \item We empirically demonstrate up to \textbf{6.83$\times$} lossless speedup across seven checkpoints from two Looped Transformer families through an SGLang~\citep{zheng2024sglang} implementation with branch-wise KV management and CUDA Graph optimization, outperforming the training-based DFlash~\citep{chen2026dflash} by up to $1.7\times$ (\cref{sec:experiments} and \ref{app:sglang_impl}).
\end{itemize}

\section{Related Work}

\paragraph{Speculative Decoding.}
Speculative decoding \citep{leviathan2023fast,chen2023accelerating} is a lossless inference acceleration technique that uses a lightweight drafting mechanism to generate multiple candidate tokens, which are then verified in parallel by the target model. Drafting mechanisms typically employ either a smaller pretrained model from the same architecture family \citep{leviathan2023fast} or a specially trained external drafter. Modern state-of-the-art drafters, such as EAGLE-3 \citep{li2025eagle3} and DFlash \citep{chen2026dflash}, achieve high acceptance rates by conditioning proposals on intermediate target representations.
Self-speculative decoding instead generates draft tokens using the target model itself without a separate external drafter. 
Draft \& Verify~\citep{zhang2024draft} obtains drafts by skipping intermediate layers, and LayerSkip~\citep{elhoushi2024layerskip} utilizes early exits. 

\paragraph{Looped Transformers.}
Looped Transformers~\citep{geiping2025scaling,zhu2025scaling,mcleish2026teaching,park2026loopus,nanbeige2026nanbeige} repeatedly apply shared Transformer blocks across recurrent depths, increasing effective depth without proportionally increasing the number of parameters. While they achieve higher task accuracy than standard Transformers of comparable size, their recurrent computation incurs substantially greater memory traffic during inference, limiting decoding throughput.

Recent work reduces recurrent inference cost by modifying the model architecture or adapting the number of recurrent steps. 
LT2 \citep{deng2026lt2} substitutes full attention with linear or sparse attention variants and distills hybrid checkpoints, while Think-at-Hard \citep{fu2026think} incorporates a routing controller with depth-aware adapters to adjust recurrent passes dynamically. Rather than altering attention or adding routing heads, SPEED \citep{hooper2023speed} explores speculative pipelining across cyclically shared decoder groups. However, SPEED requires a custom training process and only supports greedy decoding.

\begin{figure}[t]
\vspace{-1.5em}
\centering
\includegraphics[width=1\linewidth]{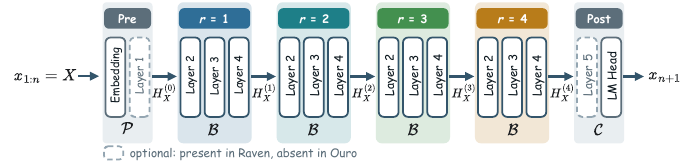}
\vspace{-1.5em}
\caption{Overview of a Looped Transformer with $R=4$ recurrent steps. In $\mathtt{Ouro}$ models~\citep{zhu2025scaling}, 
the pre-layers $\mathcal{P}$ and post-layers $\mathcal{C}$ contain only the Embedding layer and LM Head, respectively, 
% there are only the Embedding layer and LM Head in pre-layers $\mathcal{P}$ and post-layers $\mathcal{C}$, 
whereas in $\mathtt{Raven}$ models~\citep{mcleish2026teaching} 
% there are several Transformer layers in both $\mathcal{P}$ and $\mathcal{C}$
both $\mathcal{P}$ and $\mathcal{C}$ 
contain several Transformer layers.
}
\label{fig:looped-overview}
\end{figure}

\section{Preliminaries}
\label{sec:preliminaries}

\subsection{Looped Transformers}
\label{sec:prelim-looped}
For a sequence prefix $X = x_{1:n}$ of length $n$, the Looped Transformer computation can be abstracted into (1) pre-layers $\mathcal{P}$, (2) recurrent-layers $\mathcal{B}$, and (3) post-layers $\mathcal{C}$:
\begin{equation}
    \begin{aligned}
        H_{X}^{(0)} &= \mathcal{P}(X), \\
        H_{X}^{(r)} &= \mathcal{B}\!\left(H_{X}^{(r-1)}\right),
        \quad r=1,\ldots,R, \\
        p_R(\cdot\mid X) &= \mathcal{C}\!\left(H_{X}^{(R)}\right).
    \end{aligned}
    \label{eq:loop-recurrence}
\end{equation}
Here $H_{X}^{(0)}$ and $H_{X}^{(r)}$ denote hidden states before recurrence and at
recurrent depth $r$, respectively; $r$ indexes recurrent depth, and $R$ denotes
the total recurrent depth of the model. $x_{n+1} \sim p_R(\cdot\mid X)$ is the target
next-token probability distribution.
The pre-layers $\mathcal{P}$ include the token embedding, and
the post-layers $\mathcal{C}$ incorporate the LM head as well
as temperature scaling, optional top-$k$/top-$p$ filtering, and softmax
normalization to output a probability distribution over the vocabulary
$\mathcal{V}$.

More generally, applying the post-layers at any recurrent
depth $r$ yields an intermediate next-token distribution
$p_r(\cdot\mid X) = \mathcal{C}(H_{X}^{(r)})$, of which the target
distribution $p_R$ is the special case $r = R$.
Existing Looped Transformers instantiate this abstraction in different ways. $\mathtt{Ouro}$ \citep{zhu2025scaling} places all Transformer blocks inside $\mathcal{B}$, whereas $\mathtt{Raven}$ \citep{mcleish2026teaching} places only a subset of intermediate blocks in $\mathcal{B}$ and the remaining ones in $\mathcal{P}$ and $\mathcal{C}$ (\Cref{fig:looped-overview}).

\subsection{Rejection Sampling and Speculative Decoding}
\label{sec:prelim-spec}

For any two distinct probability distributions $\mu$ and $\nu$ over a
vocabulary $\mathcal{V}$, we define their element-wise positive difference as
$[\mu - \nu]_+(v) \coloneqq \max(\mu(v) - \nu(v), 0)$ for each
$v \in \mathcal{V}$, and the residual
distribution operator $(\mu \ominus \nu)$ as:
\begin{equation}
    (\mu \ominus \nu)(v) \coloneqq \frac{[\mu - \nu]_+(v)}{%
      \sum_{u\in\mathcal{V}} [\mu - \nu]_+(u)}.
    \label{eq:residual-op}
\end{equation}

The core idea of speculative decoding
\citep{leviathan2023fast,chen2023accelerating} is to draft fast and then verify via rejection sampling. We state the rule for a single token position, which is
the form \ourloop{} builds on.

Formally, let $X$ be the current prefix, and let
$p\!\left(\cdot \mid X\right)$ and $q\!\left(\cdot \mid X\right)$
denote the target and the proposal distribution of the next token,
respectively. A candidate
$\widetilde{x}$ is drawn from the
proposal distribution $q\!\left(\cdot \mid X\right)$ and accepted with probability
\begin{equation}
    \min\!\left\{1,\ \frac{p(\widetilde{x} \mid X)}{%
      q(\widetilde{x} \mid X)}\right\}.
    \label{eq:acceptance}
\end{equation}
If the candidate $\widetilde{x}$ is rejected, the next token is instead resampled from the
residual distribution:
\begin{equation}
    x \sim \Bigl(
    p\!\left(\cdot \mid X\right)
    \ominus
    q\!\left(\cdot \mid X\right)
    \Bigr).
    \label{eq:residual}
\end{equation}
\citet{leviathan2023fast} show that the token produced by
\Cref{eq:acceptance,eq:residual} is distributed exactly as a token
sampled from the target $p$, so the procedure is lossless. Crucially, 
losslessness holds for \emph{any} valid proposal distribution $q$.

\section{\ourloop{}: Pipelined Self-Speculative Decoding for Looped Transformers}
\label{sec:method}

We present \ourloop{}, a training-free self-speculative decoding framework for Looped Transformers. We first show that early recurrent states can provide effective draft proposals (\cref{sec:method-observation}). Building on this observation, we introduce a pipeline with a single proposal depth $d_1<R$, where each draft starts computation for the next token before final-depth verification (\cref{sec:method-pipelining}). 
As decoding is memory-bandwidth bound, batching does not introduce noticeable overhead.
We then extend the pipeline with a second proposal depth $d_2$, using the residual distribution and the gating mechanism to further improve the proposal quality and acceptance (\cref{sec:method-second-proposal}).

\begin{figure}[t]
\centering
\setlength{\subfigcapskip}{-5pt}
\subfigure[]{\label{fig:obs-greedy}\includegraphics[width=0.48\linewidth]{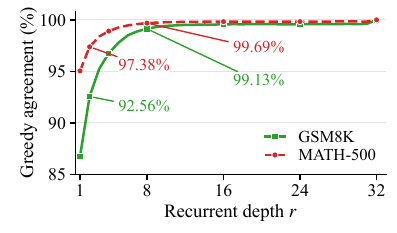}}%
\hfill
\subfigure[]{\label{fig:obs-tv}\includegraphics[width=0.48\linewidth]{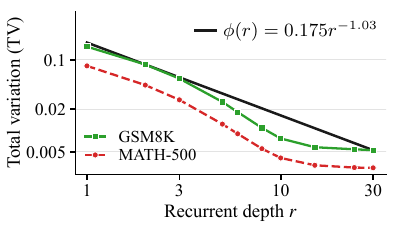}}
\setlength{\abovecaptionskip}{3pt}
\caption{Comparison between intermediate proposals at depth $r$ and the target distribution at depth $R=32$ for \ravenllama. (a)~Greedy top-1 agreement on GSM8K and MATH-500. (b)~Sampling ($T=1.0$, top-$p=0.7$) total variation (TV) distance to the target on the same benchmarks. 
TV is obtained on 10 depths over GSM8K and MATH-500, averaged per token position on top-$p$ distribution. $\phi(r)$ is the empirical power-law upper envelope for the TV distance, constructed as described in \cref{sec:method-observation}.
}
\label{fig:obs-convergence}
\end{figure}

\subsection{Observation: Early Recurrent States are Effective Drafters}
\label{sec:method-observation}

In Looped Transformers, proposals derived from intermediate recurrent states can tightly approximate the target distribution. In \cref{fig:obs-convergence}, we observe this behavior in the \ravenllama model~\citep{mcleish2026teaching}. For instance, on GSM8K \citep{cobbe2021training}, the depth-$1$ greedy agreement exceeds $86\%$ and surpasses $99\%$ by depth $8$, with a similar trend on MATH-500 \citep{hendrycks2021measuring,lewkowycz2022solving,kydlicek2025fixing,lightman2024lets}. 
We further estimate the empirical total variation (TV) distance between the intermediate readout and the target distributions under the same setting, and observe that it decreases sharply as recurrence depth $r$ increases. To characterize this, we bound the empirical TV distance with a power-law upper envelope $\phi(r)=\beta r^{-\alpha}$. Specifically, 
we choose parameters $\alpha, \beta > 0$ to minimize the maximum ratio between $\phi(r)$ and the empirical TV distance over all depths $r$. This power-law upper envelope allows us to analyze the optimal depth configuration studied in \cref{sec:theory}.

These observations motivate using early recurrent states to draft future tokens, while continuing recurrence to verify the drafts at depth $R$. Deeper intermediate states can further serve as fallback proposals when earlier predictions diverge from the target.

\subsection{Pipelined Self-Speculative Decoding with a Single Proposal}
\label{sec:method-pipelining}

We first describe the single-proposal pipeline in \Cref{fig:pipeline}(b), where each token is drafted at only one intermediate depth $d_1$ where $d_1<R$, and verified at the final depth $R$. 
Unlike vanilla decoding, which completes all $R$ recurrences before moving to the next token (\cref{fig:pipeline}(a)), this schedule overlaps computation across token positions (\cref{fig:pipeline}(b)). We call the recurrent computation for one prefix a \emph{branch}, represented as a horizontal sequence of recurrent states in the figure. A draft starts a child branch, while its parent continues toward verification. Multiple branches can remain active at different recurrent depths and crucially they are advanced by a \textit{single batched call} to the shared block.

For a branch $b$, we denote its prefix by $X^{(b)}$ and use the hidden states $H_{X^{(b)}}^{(r)}$ and readouts $p_r(\cdot\mid X^{(b)})$ from \cref{sec:prelim-looped}. We write $p_r$ when the prefix is clear. After prefilling the prompt and committing the first decoded token, \ourloop{} initializes a branch on the resulting prefix at depth $0$ and repeats the following steps:
\begin{enumerate}[leftmargin=*,label=\arabic*.]
    \item \textbf{Batched Recurrence.} All active branches advance by one recurrent depth through a single batched call to $\mathcal{B}$. In \Cref{fig:pipeline}(b), each column represents one such step across the active branches.
    \item \textbf{Early Drafting.} When branch $b$ reaches $d_1$, it reads out $q_1\coloneqq p_{d_1}=\mathcal{C}(H_{X^{(b)}}^{(d_1)})$ and samples $\widetilde{x}_b^{(1)}\sim q_1$. A child branch is then initialized at depth $0$ with the extended prefix $X^{(b)}\mathbin{\Vert}\widetilde{x}_b^{(1)}$. For example, with $d_1=1$ in \Cref{fig:pipeline}(b), the draft ``France'' at timestep 1 starts a child at timestep 2 before ``France'' is verified at timestep 4.
    \item \textbf{Verification and Continuation.} When the branch reaches $R$, it verifies $\widetilde{x}_b^{(1)}$ against $p_R$ using rejection sampling as \cref{eq:acceptance}. If accepted, the draft is committed and decoding continues from its child, which has already completed $R-d_1$ recurrent steps.
    In \Cref{fig:pipeline}(b), accepting ``and'' at timestep 8 preserves its child branch that started at timestep 6. If rejected, a replacement $x_b\sim p_R\ominus q_1$ is committed, all speculative descendants of $b$ are pruned, and a new branch starts on $X^{(b)}\mathbin{\Vert}x_b$ at depth $0$. The ``France'' draft at timestep 1 in \Cref{fig:pipeline}(b) illustrates this case: depth-$R$ commits ``Paris", which differs from the draft ``France", so we prune the descendants of ``France" and restart computation for the next token.   
\end{enumerate}

With sustained draft acceptance, tokens are committed every $d_1$ recurrent steps rather than every $R$ steps. However, each rejection discards the speculative progress and requires another $R$ recurrent steps before the next commitment. This motivates a more accurate second proposal that can provide a fallback draft when the first draft is rejected.

\subsection{Residual Second Proposal with Gating Mechanism}
\label{sec:method-second-proposal}

We now extend the single-proposal pipeline in \cref{sec:method-pipelining} by adding a second proposal depth $d_2$, with $d_1<d_2<R$. At depth $d_2$, a branch can draft an alternative token for the same position, starting a fallback child alongside the first child. 
Accepting this fallback preserves its progress instead of restarting the pipeline. We next describe how to construct the fallback proposal, when to create it, and how to verify the two candidates.

\paragraph{Residual Second Proposal.}
The second proposal is used only after the first proposal is rejected, so the required target distribution is $p_R\ominus q_1$. We therefore draft from an estimate of this residual target. When branch $b$ reaches $d_2$, it obtains a deeper readout $\widetilde{q}_2\coloneqq p_{d_2}=\mathcal{C}(H_{X^{(b)}}^{(d_2)})$ and forms
\begin{equation}
    q_2\coloneqq\widetilde{q}_2\ominus q_1
    =p_{d_2}\ominus p_{d_1},\qquad
    \widetilde{x}_b^{(2)}\sim q_2.
    \label{eq:residual-proposal}
\end{equation}
Here we use $\widetilde{q}_2$ as an estimator for $p_R$ and sample the second proposal from the residual distribution $q_2$. A second child then starts from depth $0$ with the prefix $X^{(b)}\mathbin{\Vert}\widetilde{x}_b^{(2)}$. The two proposals represent alternative branches for the same token position, as illustrated by the ``France'' and ``Paris'' drafts at timesteps 1 and 2 in the top row of \cref{fig:pipeline}(c). {Our ablation study in \cref{sec:ablation-q2-strategies} shows that the residual second proposal $q_2$ consistently improves the decoding speedup over the non-residual second proposal $\widetilde{q}_2$ across various benchmarks.}

\paragraph{Gating Mechanism.}
Creating a fallback at every branch would rapidly increase the number of active branches, since each child can itself issue two proposals. 
Specifically, if every active branch forks at both $d_1$ and $d_2$, the number of new branches $s_k$ spawned in the $k$-th interval of $d_1$ recurrent steps obeys the delayed recurrence $s_k = s_{k-1} + s_{k - d_2/d_1}$, where $s_k = 1$ for $0 \le k < d_2/d_1$\footnote{Throughout the paper, we consider that $d_2$ is divisible by $d_1$.}. For $R=32$, $d_1=2$, and $d_2=8$, this gives maximum branch count $B = \sum_{k=0}^{R/d_1 - 1} s_k =249$. To limit this growth of $B$, \ourloop{} opens the gate of second proposal and creates a second proposal branch only if
\begin{equation}
    \widetilde{q}_2(\widetilde{x}_b^{(1)})
    <q_1(\widetilde{x}_b^{(1)}).
    \label{eq:adaptive-gate}
\end{equation}
Intuitively, a fallback branch is created when the deeper recurrent state withdraws confidence from the primary candidate. Under greedy decoding, the rule reduces to $\arg\max\widetilde{q}_2\ne\widetilde{x}_b^{(1)}$. In \Cref{fig:pipeline}(c), the gate stays closed for the ``and'' branch where no second proposal is triggered at timestep 4. 
\Cref{sec:ablation-q2-strategies} evaluates the benefit of the gating mechanism. Moreover, \cref{app:gating} examines the gate opening frequency and its effect on the effective batch size.

\paragraph{Cascade Verification.}
When a branch $b$ reaches depth $R$, it produces the exact target distribution $p_R$ and a cascade verification attempts to accept the primary candidate, then the fallback candidate, and restarts the pipeline only if neither candidate is accepted.
\begin{enumerate}[leftmargin=*,label=\arabic*.]
    \item \textbf{First Proposal.} Accept $\widetilde{x}_b^{(1)}$ against $p_R$ with probability $\min\!\left\{1, p_R(\widetilde{x}_b^{(1)})/q_1(\widetilde{x}_b^{(1)})\right\}$. If accepted, commit it and keep its descendants, and prune the second child's subtree if present. In \Cref{fig:pipeline}(c), at timestep 7, accepting the first draft ``it'' from timestep 4 preserves its continuation while pruning the second draft ``the'' from timestep 5 and its subtree.

    \item \textbf{Second Proposal.} If the first proposal is rejected and a second proposal exists, verify $\widetilde{x}_b^{(2)}$ against $\rho_1\coloneqq p_R\ominus q_1$ with probability $\min\!\left\{1, \rho_1(\widetilde{x}_b^{(2)})/q_2(\widetilde{x}_b^{(2)})\right\}$. If accepted, commit it and keep its descendants, and prune the first child's subtree. This child has already completed $R-d_2$ recurrent steps, so the next commitment requires only $d_2$ more steps instead of $R$. This is the rejection of ``France'' and acceptance of the second draft ``Paris'' in the top 5 rows of \cref{fig:pipeline}(c).

    \item \textbf{Residual Resampling.} If both proposals are rejected, sample $x_b\sim\rho_1\ominus q_2$. If the first is rejected and the gate was closed, sample $x_b\sim\rho_1$ instead. Commit $x_b$, prune the speculative descendants of $b$, and start a fresh branch on $X^{(b)}\mathbin{\Vert}x_b$ at depth $0$.
\end{enumerate}
Under greedy decoding, this cascade checks whether $\widetilde{x}_{b}^{(1)}$ matches $\arg\max p_R$, falling back to $\widetilde{x}_{b}^{(2)}$ if available, and committing $\arg\max p_R$ otherwise.
The algorithm is provided in \cref{app:algorithm}.

\section{Theoretical Analysis} \label{sec:theory}

In this section, we provide theoretical analysis of \ourloop{} from two perspectives. First, we show that the restart probability depends only on the approximation error of the second proposal and is independent of that of the first proposal. 
Second, we use the power-law decay of approximation error across recurrent depths, as observed in \cref{fig:obs-tv}, to derive the optimal proposal depths that minimize the expected number of recurrent steps per generated token.
The optimal proposal depths derived from our analysis closely match the empirically optimal depth configurations.

\begin{theorem}[Rejection Rate of \ourloop{}]
\label{thm:flushing}
Fix proposal depths $1\leq d_1 < d_2 < R$ and 
% let $\varepsilon_2 := \tv{p_{d_2}, p_R}$. 
let $\varepsilon_2$ be the total variation distance between $p_{d_2}$ and $p_R$.
Under the residual second proposal $q_2 = p_{d_2} \ominus p_{d_1}$ and cascade verification, the probability that both candidates are rejected at a given position is at most $\varepsilon_2$, independently of $d_1$. Moreover, with the gating mechanism, the bound is at most $2 \varepsilon_2$.
\end{theorem}

Proofs of all theorems are provided in \cref{app:proofs}.
\Cref{thm:flushing} implies that the residual proposal $p_{d_2}\ominus p_{d_1}$ achieves the same rejection bound $\varepsilon_2$ as a proposal drawn directly from $p_{d_2}$, while avoiding any additional dependence on $d_1$.
Thus, introducing the residual proposal does not worsen the rejection bound compared with using \(p_{d_2}\) directly.
With the gating mechanism described in \cref{sec:method-second-proposal}, the bound increases only by a constant factor, from $\varepsilon_2$ to $2\varepsilon_2$.

This result allows us to characterize the cost of a rejection using only the approximation quality at the second proposal depth $d_2$. Hence, we can focus on how the approximation error decreases with recurrent depth to derive the proposal depths that minimize the expected decoding cost. Motivated by the observations in \cref{fig:obs-tv}, we assume a non-increasing power-law envelope $\phi(r) = \beta r^{-\alpha}$ that upper bounds the expected total variation distance between the intermediate readout distribution $p_r$ and the target distribution $p_R$.

\begin{theorem}[Proposal Depth Selection]
\label{thm:depths-main}
Let $X$ be a random prefix drawn from a fixed distribution over prefixes.
Assume that there exist constants $\alpha, \beta > 0$ such that
\begin{align}\label{eq:phi}
\mathbb{E}_{X}\left[ \tv{p_r(\cdot \mid X), p_R(\cdot \mid X)}\right] \leq \beta r^{-\alpha}.
\end{align}
for every $r = 1,...,R$ and the total depth $R > \frac{1}{2\alpha}\max\{(\alpha\beta)^{1/\alpha},(\alpha\beta)^{-\alpha-1},(\alpha\beta)^{1/\alpha}(2\alpha)^{\frac{\alpha^2 + \alpha+1}{\alpha^2}}\}$.
Then, the optimal proposal depths to minimize expected number of recurrent steps to commit a single token are
\begin{align}
    d_1^\star = (\alpha\beta)^{\frac{\alpha+1}{\alpha^2+\alpha+1}} (2\alpha R)^{\frac{1}{\alpha^2+\alpha+1}}, \qquad d_2^\star = (\alpha \beta)^{\frac{\alpha}{\alpha^2 + \alpha + 1}}(2\alpha R)^{\frac{\alpha +1}{\alpha^2 + \alpha +1}}. \label{eq:optimal_depths}
\end{align}
\end{theorem}

As shown in \cref{fig:obs-tv}, we obtain an empirical upper bound with $\beta=0.175$ and $\alpha=1.03$ for \ravenllama model with $R=32$. Substituting these values into \cref{eq:optimal_depths} gives us $({d_1^\star},{d_2^\star}) =(1.26, 8.85)$. 
In \cref{sec:experiments}, we explore various $(d_1, d_2)$ configurations and find that $(d_1,d_2)=(2,8)$ performs best for the same model. This demonstrates that our theoretical analysis closely captures the empirically optimal depth configuration.

\section{Experiments}
\label{sec:experiments}

\paragraph{Models and Evaluations.}
We evaluate \ourloop{} across two representative Looped Transformer families:
the $\mathtt{Ouro}$ family with $R=4$~\citep{zhu2025scaling} and the $\mathtt{Raven}$ family with $R=32$~\citep{mcleish2026teaching}.
Evaluations cover mathematical reasoning, general reasoning, and code
generation: GSM8K \citep{cobbe2021training}, MATH-500
\citep{hendrycks2021measuring,lewkowycz2022solving,kydlicek2025fixing,lightman2024lets}, BBH
\citep{suzgun2023challenging}, HumanEval+ and MBPP+
\citep{chen2021evaluating,austin2021program,liu2023evalplus} for base models,
alongside GSM8K-CoT, MATH-500, AIME 2024, and AIME 2025
\citep{maa2024aime,maa2025aime} for reasoning models (i.e., with thinking mode enabled).
Base models are evaluated under greedy decoding ($T=0.0$), whereas reasoning
models are evaluated using chat templates with thinking modes enabled, under
both greedy decoding ($T=0.0$) and sampling
($T=1.0, \text{top-}p=0.7$).
All evaluations are conducted in a single-batch setting integrated with
$\mathtt{lm\text{-}evaluation\text{-}harness}$ \citep{gao2024lmharness}.
Detailed checkpoint identifiers and benchmark configurations are provided in \cref{app:exp-details}.

\paragraph{Implementation and Metrics.}
We implement \ourloop{} by extending the SGLang serving engine
\citep{zheng2024sglang}.
All experiments are executed on an NVIDIA RTX PRO 6000 Blackwell GPU in BF16
precision using the Triton attention backend and the TinyGEMM backend provided by FlashInfer \citep{ye2025flashinfer}. Further SGLang implementation details are in \cref{app:sglang_impl}.
We report speedup on {wall-clock time} and {mean accepted length} $\gamma$.
Wall-clock time speedup is measured over the decoding phase, after the completion of prompt prefilling until the end of generation.
We define $\gamma$ as the average number of generated tokens produced per full model execution of $R$ recurrent steps:
\begin{equation}
\gamma := \frac{N_{\text{decode}} \cdot R}{%
  N_{\text{accept}, 1} \cdot d_1 + N_{\text{accept}, 2} \cdot d_2
  + (N_{\text{decode}} - N_{\text{accept}, 1} - N_{\text{accept}, 2}) \cdot R},
\end{equation}
where $N_{\text{decode}}$ is the number of generated tokens, and
$N_{\text{accept}, 1}$ and $N_{\text{accept}, 2}$ denote the number of tokens
accepted from the first and second proposals, respectively.
By construction, $\gamma \le R / d_1$, with the upper bound attained when all
tokens are accepted at the first proposal depth $d_1$.

\begin{table}[t]
\centering
\caption{\ourloop{} decoding speedup over the standard autoregressive decoding and mean accepted
length $\gamma$ on general reasoning, math, and code benchmarks across different
proposal depths. Bold values mark per-model, per-benchmark maxima,
with speedup and $\gamma$ selected independently.}
\label{tab:base-models}
\vspace{0.04in}
\fontsize{9}{11}\selectfont
\setlength{\tabcolsep}{2.2pt}
\begin{tabular*}{\linewidth}{@{\hspace{\tabcolsep}\extracolsep{\fill}}llcccccccccc@{\hspace{\tabcolsep}}}
\toprule
\multirow{2}{*}{\textbf{Models}} & \multirow{2}{*}{\makecell{\textbf{Proposal} \\ \textbf{Depths}}}
& \multicolumn{2}{c}{\textbf{GSM8K}} & \multicolumn{2}{c}{\textbf{MATH-500}}
& \multicolumn{2}{c}{\textbf{BBH}} & \multicolumn{2}{c}{\textbf{HumanEval+}}
& \multicolumn{2}{c}{\textbf{MBPP+}} \\
\cmidrule(lr){3-4} \cmidrule(lr){5-6} \cmidrule(lr){7-8} \cmidrule(lr){9-10}
\cmidrule(lr){11-12}
& & Speedup & $\gamma$ & Speedup & $\gamma$ & Speedup & $\gamma$ & Speedup & $\gamma$
& Speedup & $\gamma$ \\
\midrule
\multicolumn{12}{c}{\textbf{Greedy Setting: Temperature=0.0}} \\
\midrule
\multirow{2}{*}{\ourosmall}
  & 1   & 2.64$\times$ & 3.13 & 2.64$\times$ & 3.25 & 2.66$\times$ & 3.30
              & \textbf{3.17}$\times$ & 3.73 & 2.79$\times$ & 3.36 \\
  & 1, 2 & \textbf{2.85}$\times$ & \textbf{3.50}
              & \textbf{2.78}$\times$ & \textbf{3.55}
              & \textbf{2.79}$\times$ & \textbf{3.59}
              & 3.16$\times$ & \textbf{3.86}
              & \textbf{2.92}$\times$ & \textbf{3.64} \\
\midrule
\multirow{2}{*}{\ourolarge}
  & 1   & 3.03$\times$ & 3.43 & 2.93$\times$ & 3.50 & 2.89$\times$ & 3.49
              & \textbf{3.36}$\times$ & 3.78 & 3.00$\times$ & 3.44 \\
  & 1, 2 & \textbf{3.12}$\times$ & \textbf{3.68}
              & \textbf{2.99}$\times$ & \textbf{3.72}
              & \textbf{2.98}$\times$ & \textbf{3.72}
              & 3.33$\times$ & \textbf{3.89}
              & \textbf{3.09}$\times$ & \textbf{3.69} \\
\midrule
\multirow{5}{*}{\makecell[l]{$\mathtt{Raven}\text{-}$\\$\mathtt{Llama}\text{-}\mathtt{3.2}$}}
  & 2   & 3.04$\times$ & 7.62 & 4.77$\times$ & 11.56 & 4.19$\times$ & 10.35
              & 4.93$\times$ & 11.73 & 3.96$\times$ & 9.70 \\
  & 4   & 4.06$\times$ & 6.55 & 4.76$\times$ & 7.43 & 4.57$\times$ & 7.21
              & 4.89$\times$ & 7.52 & 4.55$\times$ & 7.14 \\
  & 4, 8 & 4.33$\times$ & 7.51 & 4.90$\times$ & 7.83 & 4.76$\times$ & 7.74
              & 5.01$\times$ & 7.84 & 4.72$\times$ & 7.70 \\
  & 2, 10 & 4.38$\times$ & 12.02 & 5.70$\times$ & 14.40
               & 5.32$\times$ & 13.72 & 5.83$\times$ & 14.36 & 5.07$\times$ & 13.08 \\
  & 2, 8 & \textbf{4.49}$\times$ & \textbf{12.50}
              & \textbf{5.82}$\times$ & \textbf{14.60}
              & \textbf{5.39}$\times$ & \textbf{13.98}
              & \textbf{5.88}$\times$ & \textbf{14.55}
              & \textbf{5.14}$\times$ & \textbf{13.44} \\
\midrule
\multirow{5}{*}{\makecell[l]{$\mathtt{Raven}$-\\$\mathtt{OLMo}\text{-}\mathtt{2}\text{-}\mathtt{0425}$}}
  & 2   & 3.75$\times$ & 8.95 & 3.16$\times$ & 8.34 & 3.52$\times$ & 9.24
              & 5.05$\times$ & 12.12 & 5.04$\times$ & 12.06 \\
  & 4   & 4.38$\times$ & 6.97 & 4.04$\times$ & 6.94 & 4.17$\times$ & 7.07
              & 4.78$\times$ & 7.59 & 4.81$\times$ & 7.62 \\
  & 4, 8 & 4.60$\times$ & 7.67 & 4.26$\times$ & 7.67 & 4.38$\times$ & 7.73
              & 4.86$\times$ & 7.87 & 4.89$\times$ & 7.88 \\
  & 2, 10 & 5.06$\times$ & 12.87 & \textbf{4.29}$\times$ & 12.65
               & 4.57$\times$ & 13.17 & 5.75$\times$ & 14.37 & 5.86$\times$ & 14.56 \\
  & 2, 8 & \textbf{5.20}$\times$ & \textbf{13.35}
              & \textbf{4.29}$\times$ & \textbf{13.11}
              & \textbf{4.66}$\times$ & \textbf{13.60}
              & \textbf{5.86}$\times$ & \textbf{14.74}
              & \textbf{5.97}$\times$ & \textbf{14.86} \\
\midrule
\multirow{5}{*}{\makecell[l]{$\mathtt{Raven}\text{-}$\\$\mathtt{TinyLlama}\text{-}\mathtt{3T}$}}
  & 2   & 3.86$\times$ & 8.18 & \multicolumn{4}{c}{\multirow{5}{*}{\makecell{prompts exceed \\ max context length}}}
              & 5.85$\times$ & 11.94 & 4.60$\times$ & 9.69 \\
  & 4   & 4.69$\times$ & 6.72 & \multicolumn{4}{c}{}
              & 5.40$\times$ & 7.46 & 4.98$\times$ & 7.08 \\
  & 4, 8 & 5.01$\times$ & 7.58 & \multicolumn{4}{c}{}
              & 5.58$\times$ & 7.85 & 5.18$\times$ & 7.71 \\
  & 2, 10 & 5.38$\times$ & 12.43 & \multicolumn{4}{c}{}
               & \textbf{6.83}$\times$ & 14.49
               & \textbf{5.96}$\times$ & 13.37 \\
  & 2, 8 & \textbf{5.53}$\times$ & \textbf{12.89} & \multicolumn{4}{c}{}
              & 6.81$\times$ & \textbf{14.70}
              & 5.76$\times$ & \textbf{13.55} \\
\bottomrule
\end{tabular*}%

\end{table}

\subsection{Base Models}
\cref{tab:base-models} presents the decoding speedup and mean accepted length $\gamma$ across base model checkpoints.
Across various tasks, \ourloop{} provides consistent speedups over standard
autoregressive decoding.
For the $\mathtt{Ouro}$ family ($R=4$), \ourloop{} achieves speedups up to $3.36\times$,
with $\gamma$ exceeding $3.5$ out of $R/d_1=4$.
For the deeper $\mathtt{Raven}$ family ($R=32$), \ourloop{} yields peak speedups ranging
from $5.88\times$ to $6.83\times$ across models, with $\gamma$ exceeding $14.5$. 
Here, while a single proposal at depth $d_1=2$ suffers from lower acceptance rates, adding a second proposal (e.g., $d_1=2, d_2=8$) compensates for early errors, allowing the first proposal depth to be pushed down to $d_1=2$ to enlarge the pipeline depth and maximize speedup. In \cref{app:depth-accept}, we further show that the first proposal accounts for over $86\%$ of committed tokens, while the second proposal recovers the majority of the remainder, leaving at most $2.95\%$ of all committed tokens to full pipeline restart.

Notably, $\gamma$ measures the reduction in serial recurrent steps but excludes
the cost of $\mathcal{P}$ and $\mathcal{C}$ computation.
A smaller $d_1$ triggers more frequent drafting and branch initialization,
increasing these costs, particularly in Raven models with heavy pre- and post-layers.
This overhead can outweigh the recurrent step savings, explaining why some configurations on $\mathtt{Raven}$ models achieve higher $\gamma$ values but lower wall-clock time speedup. 

\subsection{Reasoning Models}
\begin{table}[t]
\centering
\caption{\ourloop{} decoding speedup over the standard autoregressive decoding and mean accepted
length $\gamma$ with thinking enabled across different proposal depths. 
Results under temperature $0.0$ and $1.0$ (with top-$p=0.7$) are reported
separately. Bold values mark per-model, per-benchmark maxima, with speedup and $\gamma$ selected independently.}
\vspace{0.04in}
\label{tab:reasoning-models}
\fontsize{9}{11}\selectfont
\setlength{\tabcolsep}{5pt}
\begin{tabular*}{\linewidth}{@{\hspace{\tabcolsep}\extracolsep{\fill}}llcccccccc@{\hspace{\tabcolsep}}}
\toprule
\multirow{2}{*}{\textbf{Models}} & \multirow{2}{*}{\makecell{\textbf{Proposal} \\ \textbf{Depths}}}
& \multicolumn{2}{c}{\textbf{GSM8K}} & \multicolumn{2}{c}{\textbf{MATH-500}}
& \multicolumn{2}{c}{\textbf{AIME 2024}} & \multicolumn{2}{c}{\textbf{AIME 2025}} \\
\cmidrule(lr){3-4} \cmidrule(lr){5-6} \cmidrule(lr){7-8} \cmidrule(lr){9-10}
& & Speedup & $\gamma$ & Speedup & $\gamma$ & Speedup & $\gamma$ & Speedup & $\gamma$ \\
\midrule
\multicolumn{10}{c}{\textbf{Greedy Setting: Temperature=0.0}} \\
\midrule
\multirow{2}{*}{\ourosmallthinking}
  & 1   & 2.47$\times$ & 2.95 & 2.49$\times$ & 3.12
              & \textbf{2.56}$\times$ & 3.26
              & 2.45$\times$ & 3.23 \\
  & 1, 2 & \textbf{2.70}$\times$ & \textbf{3.36}
              & \textbf{2.65}$\times$ & \textbf{3.50}
              & 2.52$\times$ & \textbf{3.54}
              & \textbf{2.59}$\times$ & \textbf{3.55} \\
\midrule
\multirow{2}{*}{\ourolargethinking}
  & 1   & 2.76$\times$ & 3.21 & 2.63$\times$ & 3.33 & 2.53$\times$ & 3.31
              & 2.58$\times$ & 3.22 \\
  & 1, 2 & \textbf{2.95}$\times$ & \textbf{3.55}
              & \textbf{2.78}$\times$ & \textbf{3.64}
              & \textbf{2.58}$\times$ & \textbf{3.63}
              & \textbf{2.66}$\times$ & \textbf{3.64} \\
\midrule
\multicolumn{10}{c}{\textbf{Sampling Setting: Temperature=1.0}} \\
\midrule
\multirow{2}{*}{\ourosmallthinking}
  & 1   & 2.11$\times$ & 2.64 & 2.17$\times$ & 2.88 & 2.07$\times$ & 2.77
              & 2.05$\times$ & 2.76 \\
  & 1, 2 & \textbf{2.42}$\times$ & \textbf{3.21}
              & \textbf{2.41}$\times$ & \textbf{3.35}
              & \textbf{2.34}$\times$ & \textbf{3.31}
              & \textbf{2.27}$\times$ & \textbf{3.28} \\
\midrule
\multirow{2}{*}{\ourolargethinking}
  & 1   & 2.62$\times$ & 3.11 & 2.52$\times$ & 3.21 & 2.37$\times$ & 3.14
              & 2.39$\times$ & 3.10 \\
  & 1, 2 & \textbf{2.83}$\times$ & \textbf{3.54}
              & \textbf{2.70}$\times$ & \textbf{3.61}
              & \textbf{2.51}$\times$ & \textbf{3.54}
              & \textbf{2.65}$\times$ & \textbf{3.53} \\
\bottomrule
\end{tabular*}%

\end{table}

\cref{tab:reasoning-models} reports the performance of \ourloop{} in
reasoning models with the {thinking} mode enabled.
Across long-context chain-of-thought generation on GSM8K, MATH-500, and AIME 2024/2025 benchmarks, \ourloop{} delivers speedups up to $2.95\times$ with high $\gamma$ values even on complex, long-reasoning benchmarks.

For stochastic sampling ($T=1.0, \text{top-}p=0.7$), \ourloop{} maintains competitive speedups up to $2.83\times$. Here, the second proposal proves particularly effective: drafting from the residual distribution upon rejection increases $\gamma$ from $3.11$ to $3.54$ on \ourolargethinking{} under the GSM8K benchmark, yielding substantial speedup gains over using only the first proposal.

\subsection{Ablation Study}
\label{sec:ablations}

\paragraph{Residual Second Proposal and Gating Mechanism.}
\label{sec:ablation-q2-strategies}
% \begin{table}[t]
% \centering
% \caption{Ablation study of the residual and gated second proposal in \ourloop{}.}
% \label{tab:second-proposal-ablation}
% \small
% \input{tables/second-proposal-ablation}
% \end{table}

\begin{table}[t]
\centering
\vspace{-1em}
\caption{
% Ablation study of the residual and gated second proposal in \ourloop{}.
Ablation study of different second proposal methods in \ourloop{}.
% \textcolor{red}{Naive (i.e., non-residual) and Residual Proposal sample their second proposals from $\widetilde{q}_2$ and ${q}_2$, respectively, as defined in \cref{eq:residual-proposal}.}
%\textcolor{red}{Gated Second Proposal applies the gating method described in \cref{sec:method-second-proposal} on top of Residual Second Proposal.}
% \sehoon{Ablation study of different second-proposal methods in \ourloop{}. Non-residual and Residual Proposal sample from ${q}_2$ and $\widetilde{q}_2$, respectively, as defined in \Cref{eq:residual-proposal}. Gated Second Proposal applies the gating method described in \Cref{sec:method-second-proposal} on top of Residual Second Proposal.}
% \sehoon{I suggest that we name First and Second Proposal -> Non-residual Second Proposal instead}
}
\vspace{0.04in}
\label{tab:second-proposal-ablation}
\fontsize{9}{11}\selectfont
\setlength{\tabcolsep}{5pt}
% \begin{tabular}{lcccc}
% \toprule
% \textbf{Raven-Llama-3.2} & \textbf{GSM8K-CoT} & \textbf{MATH-500} & \textbf{HumanEval+} & \textbf{MBPP+} \\
% \midrule
% Baseline & 1.00$\times$ & 1.00$\times$ & 1.00$\times$ & 1.00$\times$ \\
% \quad + First \& Second Proposal & 3.24$\times$ & 4.49$\times$ & 4.59$\times$ & 4.07$\times$ \\
% \quad + Residual Second Proposal & 3.52$\times$ & 4.73$\times$ & 4.76$\times$ & 4.57$\times$ \\
% \quad + Gated Second Proposal & \textbf{3.74}$\times$ & \textbf{4.85}$\times$ & \textbf{5.00}$\times$ & \textbf{4.71}$\times$ \\
% \bottomrule
% \end{tabular}%

\begin{tabular*}{\linewidth}{@{\hspace{\tabcolsep}\extracolsep{\fill}}lcccc@{\hspace{\tabcolsep}}}
\toprule
\textbf{Model:} \ravenllama & \textbf{GSM8K} & \textbf{MATH-500} & \textbf{HumanEval+} & \textbf{MBPP+} \\
\midrule
Standard Autoregressive Decoding & 1.00$\times$ & 1.00$\times$ & 1.00$\times$ & 1.00$\times$ \\
\quad + First \& Non-residual Second Proposal & 3.24$\times$ & 4.49$\times$ & 4.59$\times$ & 4.07$\times$ \\
\quad + Residual Second Proposal & 3.52$\times$ & 4.73$\times$ & 4.76$\times$ & 4.57$\times$ \\
\quad + Gating Mechanism & \textbf{3.74}$\times$ & \textbf{4.85}$\times$ & \textbf{5.00}$\times$ & \textbf{4.71}$\times$ \\
\bottomrule
\end{tabular*}%

\end{table}

\cref{tab:second-proposal-ablation} presents the results of an ablation
study on the second proposal in \ourloop{}, which are the residual second proposal and the gating mechanism. The ablation study is conducted on \ravenllama model under sampling ($T=1.0, \text{top-}p=0.7$) on GSM8K, MATH-500, HumanEval+, and MBPP+ benchmarks. 
We mask out the first proposal token $\widetilde{x}_b^{(1)}$ from $\widetilde{q}_2$ for the non-residual second proposal and from $q_2$ for the residual second proposal, and renormalize the corresponding distribution before sampling. This ensures that the second proposal differs from the first proposal. With gating, this masking can be omitted for the residual second proposal because the gate opens only when $\widetilde{q}_2(\widetilde{x}_b^{(1)}) < q_1(\widetilde{x}_b^{(1)})$, in which case $q_2(\widetilde{x}_b^{(1)})$ is already zero.

Adding the first and non-residual second proposal at $d_1=2, d_2=8$ yields speedups ranging from $3.24\times$ to $4.59\times$. Replacing the non-residual second proposal with the residual second proposal and adding the gating mechanism consistently improve the speedups to $3.74\times$ to $5.00\times$, showing the effectiveness of the residual second proposal and gating mechanism.
\cref{app:gating} further examines how often the gate opens and how gating reduces the effective batch size during decoding.

\paragraph{Comparison between \ourloop{} and DFlash.}
\label{sec:dflash-comparison}
\begin{figure}[t]
\centering
\includegraphics[width=\linewidth]{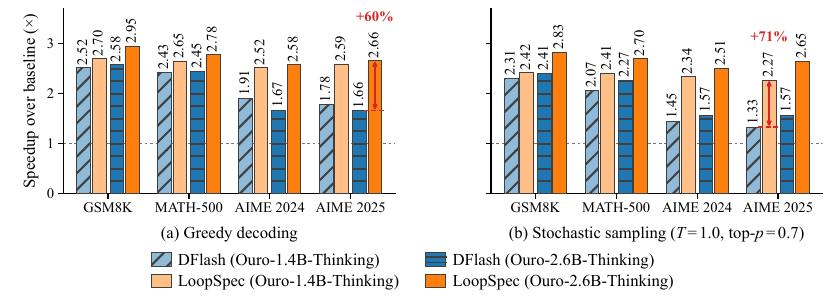}
\setlength{\abovecaptionskip}{3pt}
\vspace{-1.5em}
\caption{Decoding speedup over the 
%\textcolor{red}{autoregressive baseline} 
standard autoregressive decoding for \ourloop{} versus DFlash on the $\mathtt{Ouro}$ thinking models, under greedy decoding and stochastic sampling.}
\label{fig:dflash-comparison}
\end{figure}

We additionally compare our training-free \ourloop{} to DFlash~\citep{chen2026dflash} on various reasoning tasks, where DFlash is an effective training-based speculative decoding method that has been largely adopted in standard autoregressive models, e.g., \cite{mimo2026v25pro_fp4dflash}. 
As shown in \Cref{fig:dflash-comparison}, \ourloop{} consistently outperforms DFlash by showing up to $1.7\times$ speedup on decoding time under both greedy and stochastic sampling settings. More details on implementation and settings are provided in \cref{sec:dflash-detail}.

\section{Conclusion}
We introduced \ourloop{}, a training-free self-speculative decoding framework for Looped Transformers that converts recurrent depth into token-position parallelism. An early recurrent state drafts the next token and immediately begins computing it while the current token continues toward final-depth verification. Furthermore, we proposed a residual second proposal with a gating mechanism from a deeper recurrent depth alongside a cascade verification that keeps decoding lossless.
Across the $\mathtt{Ouro}$ ($R=4$) and $\mathtt{Raven}$ ($R=32$) families, \ourloop{} achieves speedups of up to $6.83\times$ on general reasoning, math, and coding benchmarks.

% ICLR 2027 requires an AI use statement in the submitted paper. Replace this
% comment with a truthful disclosure before submission; it does not count toward
% the page limit. See the ICLR 2027 AI Policy for Authors.
% \subsection*{AI Use Statement}
% In this work, we used generative AI tools for [...]. We have reviewed all
% AI-assisted work and take responsibility for the final content.

% Optional sections that do not count toward the page limit can be placed here:
% \subsection*{Ethics Statement}
% \subsection*{Reproducibility Statement}
% \subsection*{Acknowledgments}

% Uncomment these two lines after adding citations and entries to references.bib.
\bibliography{references}
\bibliographystyle{iclr2027_conference}

\appendix
\newpage

\section{Experimental Setup Details}
\label{app:exp-details}

\cref{tab:model-mapping} presents the complete mapping between evaluated
Looped Transformer checkpoints and the model names used throughout this paper.
\cref{tab:benchmark-specs} outlines the benchmark specifications,
prompting protocols, and maximum generation token budgets
across base and reasoning models.
For reasoning models, evaluations apply chat templates with thinking mode
enabled to support long chain-of-thought derivations.
For code generation on HumanEval+ and MBPP+, we employ zero-shot custom prompts that instruct models to produce self-contained Python scripts within markdown code blocks.

\begin{table}[h]
\centering
\caption{Mapping between evaluated checkpoints and the model names used in this paper.}
\label{tab:model-mapping}
\small
% \begin{tabular}{ll}
% \toprule
% \textbf{Paper Name} & \textbf{Hugging Face Checkpoint Name} \\
% \midrule
% Ouro-1.4B & \texttt{ByteDance/Ouro-1.4B} \\
% Ouro-2.6B & \texttt{ByteDance/Ouro-2.6B} \\
% Ouro-1.4B-Thinking & \texttt{ByteDance/Ouro-1.4B-Thinking} \\
% Ouro-2.6B-Thinking & \texttt{ByteDance/Ouro-2.6B-Thinking} \\
% Raven-Llama-3.2 & \texttt{smcleish/Recurrent-Llama-3.2-train-recurrence-32} \\
% Raven-OLMo-2-0425 & \texttt{smcleish/Recurrent-OLMo-2-0425-train-recurrence-32} \\
% Raven-TinyLlama-3T & \texttt{smcleish/Recurrent-TinyLlama-3T-train-recurrence-32} \\
% \bottomrule
% \end{tabular}%

\begin{tabular}{ll}
\toprule
\textbf{Name in Paper} & \textbf{Hugging Face Checkpoint} \\
\midrule
\multicolumn{2}{l}{\textit{Base Models}} \\
Ouro-1.4B & \texttt{ByteDance/Ouro-1.4B} \\
Ouro-2.6B & \texttt{ByteDance/Ouro-2.6B} \\
Raven-Llama-3.2 & \texttt{smcleish/Recurrent-Llama-3.2-train-recurrence-32} \\
Raven-OLMo-2-0425 & \texttt{smcleish/Recurrent-OLMo-2-0425-train-recurrence-32} \\
Raven-TinyLlama-3T & \texttt{smcleish/Recurrent-TinyLlama-3T-train-recurrence-32} \\
\midrule
\multicolumn{2}{l}{\textit{Reasoning Models}} \\
Ouro-1.4B-Thinking & \texttt{ByteDance/Ouro-1.4B-Thinking} \\
Ouro-2.6B-Thinking & \texttt{ByteDance/Ouro-2.6B-Thinking} \\
\bottomrule
\end{tabular}

\end{table}

\begin{table}[h]
\centering
\caption{Evaluation benchmark configurations for base and reasoning models.}
\label{tab:benchmark-specs}
\small
\begin{tabular}{llcc}
\toprule
\textbf{Benchmark} & \textbf{Prompting Protocol} & \makecell{\textbf{Base} \\ \textbf{Max Gen Toks}} & \makecell{\textbf{Reasoning} \\ \textbf{Max Gen Toks}} \\
\midrule
GSM8K      & 3-shot CoT             & 512   & 4,096  \\
MATH-500   & 4-shot                 & 2,048 & 8,192  \\
BBH        & 3-shot CoT             & 1,024 & --     \\
HumanEval+ & 0-shot (Custom Prompt) & 1,024 & --     \\
MBPP+      & 0-shot (Custom Prompt) & 1,024 & --     \\
AIME 2024  & 0-shot                 & --    & 16,384 \\
AIME 2025  & 0-shot                 & --    & 16,384 \\
\bottomrule
\end{tabular}%
\end{table}

\section{Implementation Details of Pipelined Decoding}
\label{app:sglang_impl}
We implement \ourloop{} on top of SGLang \citep{zheng2024sglang}. In this section, we cover key implementation details as below.

\paragraph{KV Cache Management.}
The biggest challenge during \ourloop{} decoding is correctly maintaining the KV states. Computation at different recurrent depths requires different KV states, even though they share model parameters. Moreover, pruning a rejected branch must preserve the prefix KV states still needed by other branches. We maintain the KV cache for each branch separately at each recurrent depth. A newly created branch reuses its parent's cached prefix by copying the KV indices rather than the KV states and stores new KV separately. After verification, we keep the selected branch and its descendants and release only KV states no longer needed by the remaining branches. For Raven, we also maintain KV caches for $\mathcal{P}$ and $\mathcal{C}$, with separate caches for $\mathcal{C}$ at each proposal and verification depth.

\paragraph{CUDA Graph Execution.}
Each decoding step involves several GPU kernels, and launching them individually adds overhead to the computation. CUDA Graphs reduce this overhead by recording a sequence of operations and replaying it with updated inputs. However, drafting and pruning change the batch size, while each recorded graph requires fixed input shapes. Since the maximum number of branches can be determined in advance, we prepare graphs for the supported batch sizes separately for $\mathcal{P}$, $\mathcal{B}$, and $\mathcal{C}$. At each step, we select the graph matching the number of branches processed by that component. Each branch's hidden state is stored between steps and loaded into the selected batch, so changing the batch size does not discard its progress. The graphs also allocate space for new KV entries and update the KV indices before model computation, avoiding separate launches for these cache-management operations.

\paragraph{Depth Configuration.}
We require $d_2$ and $R$ to be multiples of $d_1$, with $d_1 < d_2 < R$. This aligns the proposal and verification stages with the execution schedule defined by $d_1$ and synchronizes the pre-layers $\mathcal{P}$ and post-layers $\mathcal{C}$ across branches, allowing their computations to be batched.

\section{DFlash Implementation Details}
\label{sec:dflash-detail}
In \cref{sec:dflash-comparison}, we conducted an ablation study comparing \ourloop{} and DFlash. The training and inference details are as follows.

We trained 5-layer DFlash drafters with block size 16 on the $\mathtt{Ouro}$ thinking models. We used the SpecForge \citep{li2026specforge} training framework, and trained on the OpenR1-Math-220k dataset \citep{lozhkov2025openr1math220k} with regenerated answers, following the recommended recipe. We trained for 10,000 steps with a batch size of 32 on both \ourosmallthinking{} and \ourolargethinking{}. 
During inference, we use block size of 16 for DFlash, and identical hyperparameters and SGLang configurations for \ourloop{} and DFlash for fair comparison.

\section{Empirical Results on Gating Strategy and Effective Batch Size}
\label{app:gating}
Gating reduces the larger parallel computation introduced by second proposals. Each proposal spawns a child branch that continues drafting future tokens, so unconditionally issuing second proposals increases the number of active branches. We examine how often the gate opens and how gating affects the effective batch size on the two Raven models under sampling.

\paragraph{Gate Opening Frequency.}
We measure how often the gate opens at $d_2$ to trigger a second proposal across four benchmarks. As shown in \cref{tab:gate-opening}, the opening frequency ranges from $18.75\%$ to $27.10\%$ for \ravenllama and from $20.11\%$ to $40.40\%$ for \ravenolmo. These results show that gating effectively reduces additional branch creation by selectively triggering second proposals.

\begin{table}[htbp]
\centering
\caption{Gate opening frequency under sampling, measured as the percentage of gate evaluations that trigger a second proposal. }
\label{tab:gate-opening}
\small
\begin{tabular}{lrrrr}
\toprule
Model & GSM8K & MATH-500 & HumanEval+ & MBPP+ \\
\midrule
\ravenllama & $27.10\%$ & $18.75\%$ & $19.66\%$ & $20.40\%$ \\
\ravenolmo & $24.19\%$ & $40.40\%$ & $20.11\%$ & $20.49\%$ \\
\bottomrule
\end{tabular}
\end{table}

\paragraph{Effective Batch Size.}
We define effective batch size as the number of active branches processed by each recurrent forward pass and report its distribution over all recurrent forward passes during evaluation. \Cref{fig:gate-effective-batch-size} compares decoding with and without gating on GSM8K. Gating reduces the mean effective batch size from $43.75$ to $34.06$ for \ravenllama and from $40.74$ to $31.70$ for \ravenolmo. The distributions also shift toward smaller batch sizes, with observed peaks decreasing from $239$ to $176$ and from $243$ to $191$, respectively. These results show that gating  reduces both the mean and peak effective batch size during inference.

\begin{figure}[htbp]
\centering
\includegraphics[width=\linewidth]{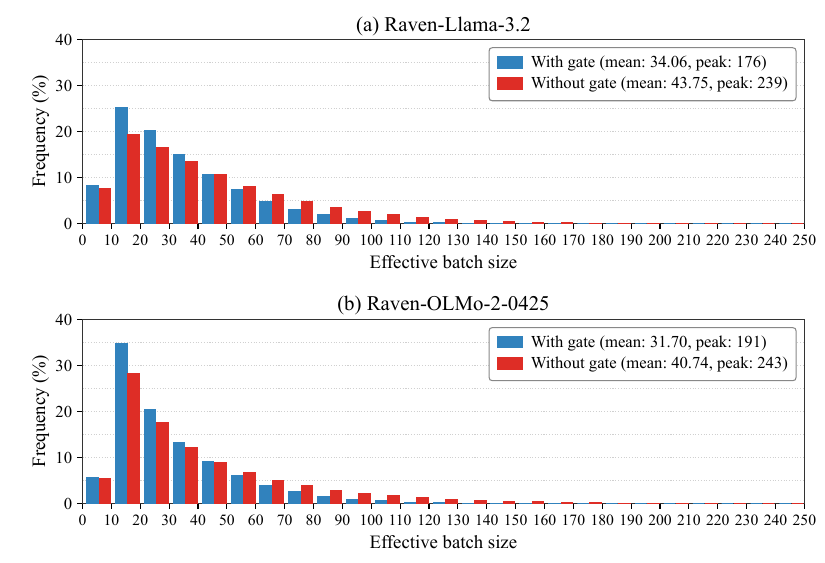}
\setlength{\abovecaptionskip}{3pt}
\vspace{-1em}
\caption{Effective batch size distributions with and without gating on GSM8K under sampling. Effective batch size is the number of active branches processed by each recurrent forward pass. Bars show the percentage of recurrent forward passes in each batch-size bin of width 10. Legends report the mean and peak batch sizes.}
\label{fig:gate-effective-batch-size}
\end{figure}

\FloatBarrier

\section{Token Commitment Sources Across Proposal Depths}
\label{app:depth-accept}
We examine how the two proposal depths contribute to token commitment in \ourloop{}. We partition committed tokens into three categories: accepted first proposals, accepted second proposals after first-proposal rejection, and full-depth fallback when neither proposal is accepted. These categories identify the source of each committed token. 
\Cref{fig:depth-accept} shows the share of tokens committed from each source for \ravenllama with $(d_1,d_2,R)=(2,8,32)$ and \ourosmall{} with $(d_1,d_2,R)=(1,2,4)$ in greedy decoding scenario for multiple benchmarks.

\paragraph{First Proposals Cover Most Tokens.}
First proposals account for $91.59\%$--$96.31\%$ of committed tokens for \ravenllama and $86.57\%$--$95.35\%$ for \ourosmall{} across the four benchmarks. Thus, the shallow proposal supplies most committed tokens in the evaluated settings, allowing the pipeline to retain computation started at the first proposal depth.

\paragraph{Second Proposals Recover Early Rejections.}
Second proposals contribute a further $3.50\%$--$7.93\%$ of tokens for \ravenllama, leaving only $0.19\%$--$0.48\%$ to full-depth fallback. For \ourosmall{}, the corresponding shares are $3.74\%$--$10.48\%$ and $0.90\%$--$2.95\%$. Across both models and all four benchmarks, second proposals therefore recover more tokens than require full-depth fallback. These results support the complementary roles of the two depths: the first proposal enables early drafting, while the second preserves speculative progress for a substantial fraction of first-proposal rejections.

\begin{figure}[htbp]
\centering
\includegraphics[width=\linewidth]{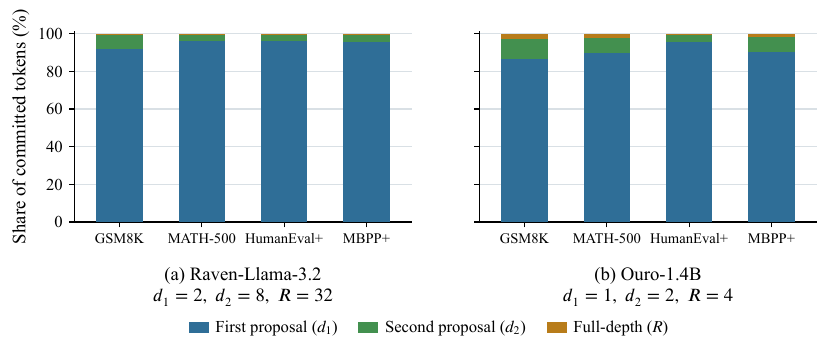}
\setlength{\abovecaptionskip}{3pt}
\vspace{-2em}
\caption{Share of committed tokens by different proposal sources.}
\label{fig:depth-accept}
\end{figure}

\FloatBarrier

\section{Empirical Power-Law Decay and Optimal Proposal Depths for Raven-OLMo-2-0425}
\label{app:olmo-power-law}

In this section, we present the \ravenolmo sampling ($T=1.0$, top-$p=0.7$) total variation (TV) distance to the target on GSM8K and MATH-500 in \cref{fig:obs-tv-olmo}.
Similar to \cref{fig:obs-tv}, we obtain the power-law envelope $\phi(r)=\beta r^{-\alpha}$, and it gives $\alpha=1.08$ and $\beta=0.207$.
With \cref{thm:depths-main}, we can compute the optimal proposal depths $(d_1^*,d_2^*) = (1.41, 9.17)$,
which again aligns with the empirical optimal depths $(d_1,d_2)=(2, 8)$ in \cref{sec:experiments}.

\begin{figure}[htbp]
\centering
\includegraphics[width=0.5\linewidth]{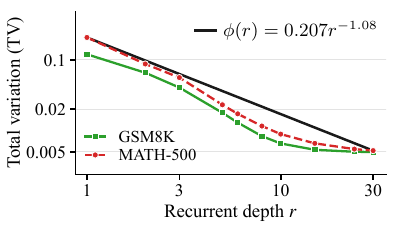}
\vspace{-1em}
\caption{Sampling ($T=1.0$, top-$p=0.7$) total variation (TV) distance between intermediate proposals at depth $r$ and the target distribution at depth $R=32$ for Raven-OLMo-2-0425 on GSM8K and MATH-500. TV is obtained on 10 depths and averaged per token position on the top-$p$ distribution. The power-law envelope $\phi(r)$ is constructed as described in \cref{sec:method-observation}.}
\label{fig:obs-tv-olmo}
\end{figure}

\FloatBarrier

\newpage
\section{Proofs}
\label{app:proofs}

We provide a formal definition of the total variation distance, which is used to evaluate the performance of speculative decoding.

\begin{definition}[Total variation distance]
    For distributions $p$ and $q$ over a finite sample space $\mathcal{X}$, the total variation distance between $p$ and $q$ is defined as 
    \begin{align*}
        \tv{p,q} := \frac12 \sum_{x \in \mathcal{X}} \abs{p(x) - q(x)} = \sum_{x \in \mathcal{X}} [p - q]_+(x) = \sum_{x \in \mathcal{X}} [q - p]_+(x).
    \end{align*}
\end{definition}

In the standard speculative decoding setting on vocabulary $\mathcal{X}$, given the target distribution $p$ and a proposal distribution $q$, the accepted probability is  
\begin{align*}
    \Pr[\mathrm{accept}] &= \sum_{x \in \mathcal{X}} q(x)\min\left(1, \frac{p(x)}{q(x)}\right)\\
    &= \sum_{x \in \mathcal{X}} \min\left( p(x), q(x)\right) \\
    &= \sum_{x \in \mathcal{X}} p(x) - [p(x) - q(x)]_{+} = 1 - \tv{p, q}.
\end{align*}
In other words, $\tv{p, q}$ measures a probability of rejection.

\subsection{Proof of \cref{thm:flushing}}

To prove \cref{thm:flushing}, we need the following lemma in advance.

\begin{lemma}
\label{lmm:tv-upper-bound}
Let $\nu_1 := p_{d_2} \ominus p_{d_1}$, $\nu_2 := p_R \ominus p_{d_1}$ and define $\varepsilon_1 := \tv{p_{d_1},p_R}$, $\varepsilon_2 := \tv{p_{d_2},p_R}$.
Suppose $\varepsilon_1 > 0$ and $p_{d_1} \ne p_{d_2}$.
Then $\tv{\nu_1,\nu_2} \leq \frac{\varepsilon_2}{\varepsilon_1}$.
\end{lemma}

\begin{proof}[Proof of \cref{lmm:tv-upper-bound}]
We write $u:=[p_R-p_{d_1}]_+$, $v:=[p_{d_2}-p_{d_1}]_+$ and $\delta := \tv{p_{d_1},p_{d_2}}$.
By the definition of the residual distribution in \cref{eq:residual-op}, $\nu_2=\frac{u}{\varepsilon_1}$, $\nu_1=\frac{v}{\delta}$. Consider two cases depending on the relative sizes of $\delta$ and $\varepsilon_1$.

\textbf{Case 1}: $\delta \leq \varepsilon_1$.
\begin{align*}
\varepsilon_1\tv{\nu_1,\nu_2}
&=
\varepsilon_1\sum_x \max(0, \nu_2(x) - \nu_1(x)) \\
&=
\sum_x \max\left(0, u(x) - \frac{\varepsilon_1}{\delta}v(x)\right) \\
&\leq
\sum_x \max(0, u(x) - v(x)) \\
&\leq
\sum_x \max(0, p_R(x) - p_{d_2}(x)) = \varepsilon_2
\end{align*}
where the last inequality holds from the fact that $[a]_+ - [b]_+ \leq [a - b]_+$. Similarly, 

\textbf{Case 2}: $\delta > \varepsilon_1$.
\begin{align*}
\varepsilon_1\tv{\nu_1,\nu_2}
&=
\varepsilon_1\sum_x \max(0, \nu_1(x) - \nu_2(x)) \\
&=
\sum_x \max\left(0, \frac{\varepsilon_1}{\delta}v(x) - u(x)\right) \\
&\leq
\sum_x \max(0, v(x) - u(x)) \\
&\leq
\sum_x \max(0, p_{d_2}(x) - p_R(x)) = \varepsilon_2
\end{align*}

By combining both cases, $\tv{\nu_1,\nu_2} \leq \frac{\varepsilon_2}{\varepsilon_1}$. This completes the proof of \cref{lmm:tv-upper-bound}.
\end{proof}

We are now ready to prove \Cref{thm:flushing}.

\begin{proof}[Proof of \Cref{thm:flushing}]
Recall that the rejection probability from the first proposal is $\varepsilon_1 = \tv{p_{d_1}, p_R}$ and that from the second one is $\tv{p_R \ominus p_{d_1}, p_{d_2} \ominus p_{d_1}}$.

If $\varepsilon_1=0$, the first proposal is always accepted, so the probability of both candidates are rejected is bounded by $\Pr[\textrm{reject}_1\land\textrm{reject}_2] \leq \Pr[\textrm{reject}_1]=\varepsilon_1=0\leq\varepsilon_2$. If $p_{d_1}=p_{d_2}$, then $\varepsilon_1=\varepsilon_2$, so the probability of both candidates are rejected is bounded by $\Pr[\textrm{reject}_1\land\textrm{reject}_2] \leq \Pr[\textrm{reject}_1] =\varepsilon_1$.

Without loss of generality, we assume that $\varepsilon_1,\delta>0$. When the gate is open, \cref{lmm:tv-upper-bound} gives
\begin{align*}
    \Pr[\textrm{reject}_1 \wedge \textrm{reject}_2 \wedge \textrm{gate open}] \leq \varepsilon_1 \cdot \frac{\varepsilon_2}{\varepsilon_1} = \varepsilon_2.
\end{align*}
% Now consider the case with the gate. 
The gate in \Cref{eq:adaptive-gate} is closed when the primary candidate $\widetilde{x}^{(1)} \sim p_{d_1}$ satisfies $p_{d_2}(\widetilde{x}^{(1)}) \geq p_{d_1}(\widetilde{x}^{(1)})$, and there is no second proposal. Under this condition, by \Cref{eq:acceptance}, the probability of drawing $x$ and rejecting it is $p_{d_1}(x) - \min\left(p_{d_1}(x), p_R(x)\right) = [p_{d_1} - p_R]_+(x)$, and hence
\begin{align*}
    \Pr[\textrm{reject}_1 \wedge \textrm{gate closed}] 
    &= \sum_{x \,:\, p_{d_2}(x) \geq p_{d_1}(x)} [p_{d_1} - p_R]_+(x) \\
    &\leq \sum_{x \,:\, p_{d_2}(x) \geq p_{d_1}(x)} [p_{d_2} - p_R]_+(x) \\
    &\leq \sum_{x} [p_{d_2} - p_R]_+(x) = \varepsilon_2,
\end{align*}
where the first inequality holds from $p_{d_1}(x) \leq p_{d_2}(x)$ under the gate-closed condition. Adding the two bounds gives $2 \varepsilon_2$. This completes the proof of \Cref{thm:flushing}.
\end{proof}

\subsection{Proof of \cref{thm:depths-main}}

Let $C$ be a random variable of the number of recurrent steps to commit a single token. Then,
\begin{align*}
    C = \begin{dcases}
        d_1 &\text{if accepted in the first proposal},\\
        d_2 &\text{else if accepted in the second proposal},\\
        R &\text{otherwise}.
    \end{dcases}
\end{align*}
The expectation of $C$ can be computed as
\begin{align}
    \mathbb{E}[C]
    &= \Pr[\textrm{accept}_1]d_1 + \Pr[\textrm{accept}_2]d_2 + (1-\Pr[\textrm{accept}_1]-\Pr[\textrm{accept}_2])R.
    \label{eq:C}
\end{align}
By the property of rejection sampling and \Cref{thm:flushing},
\begin{align*}
    &\Pr[\textrm{accept}_1] = 1-\tv{p_{d_1},p_R} = 1-\varepsilon_1 \leq 1,\\
    &\Pr[\textrm{accept}_2] \leq \Pr[\textrm{reject}_1] = \varepsilon_1,\\
    &1-\Pr[\textrm{accept}_1]-\Pr[\textrm{accept}_2] \leq 2\varepsilon_2.
\end{align*}
For a prefix $X$, the assumption in \cref{eq:phi} yields 
\begin{align*}
    &\mathbb{E}_X[\varepsilon_1] = \mathbb{E}_X\left[\tv{p_{d_1},p_R}\right] \leq \phi(d_1) = \beta d_1^{-\alpha},\\
    &\mathbb{E}_X[\varepsilon_2] = \mathbb{E}_X\left[\tv{p_{d_2},p_R}\right] \leq \phi(d_2) = \beta d_2^{-\alpha}.
\end{align*}
Substituting these bounds into \Cref{eq:C} and taking the expectation over prefix $X$, we obtain
\begin{align}
    \mathbb{E}_X[C]
    \leq 
    d_1 + \beta d_1^{-\alpha} d_2 + 2 \beta d_2^{-\alpha} R.
    \label{eq:upper-C}
\end{align}

Our goal is to find optimal $d_1$ and $d_2$ that minimize the right-hand side in \cref{eq:upper-C}. Denote the right-hand side by $f(d_1, d_2)$. The stationary conditions are obtained by taking derivative with respect to $d_1$ and $d_2$, respectively, 
\begin{align*}
\frac{\partial f}{\partial d_1} &= 1 - \alpha \beta d_1^{-\alpha -1} d_2 = 0,   \\
\frac{\partial f}{\partial d_2} &= \beta d_1^{-\alpha} - 2 \alpha \beta R d_2^{-\alpha-1} = 0.
\end{align*}
Re-writing these conditions yields
\begin{align*}
    d_1^* &= (\alpha \beta)^{\frac{\alpha + 1}{\alpha^2 + \alpha + 1}} (2 \alpha R)^{\frac{1}{\alpha^2 + \alpha + 1}}, \\
    d_2^* &= (\alpha \beta)^{\frac{\alpha}{\alpha^2 + \alpha + 1}}(2 \alpha R)^{\frac{\alpha + 1}{\alpha^2 + \alpha + 1}}. \label{eq:solution}
\end{align*}
To verify above $(d_1^*, d_2^*)$ is the minimizer, we compute the Hessian of $f$ as
\begin{align*}
    H(d_1, d_2) = \begin{bmatrix}
        \alpha (\alpha+1)\beta d_1^{-\alpha-2} d_2 & -\alpha \beta d_1^{-\alpha-1} \\
        -\alpha \beta d_1^{-\alpha-1} & 2\alpha (\alpha+1)\beta R d_2^{-\alpha-2}
    \end{bmatrix}
\end{align*}
and show that the determinant of $H(d_1^*, d_2^*)$ is strictly positive. Using the fact that $R (d_2^*)^{-\alpha-1} = (d_1^*)^{-\alpha} / (2\alpha)$, we have
\begin{align*}
    \det \left(H(d_1^*, d_2^*)\right) = \alpha \beta^2 (d_1^*)^{-2\alpha - 2} \left( \alpha^2 + \alpha + 1\right) > 0.
\end{align*}
Since this is the unique stationary point, $(d_1^*, d_2^*)$ is the global minimizer. 
Finally, the condition \[R > \frac{1}{2\alpha}\max\{(\alpha\beta)^{1/\alpha},(\alpha\beta)^{-\alpha-1},(\alpha\beta)^{1/\alpha}(2\alpha)^{\frac{\alpha^2 + \alpha+1}{\alpha^2}}\}\] ensures the validity of solution, i.e., $1 \leq d_1^* < d_2^* < R$. This completes the proof of \cref{thm:depths-main}.

\vfill
\newpage

\section{\ourloop{} Algorithm Pseudocode}
\label{app:algorithm}

\begin{algorithm}[H]
\footnotesize
\SetAlgoNlRelativeSize{-1}
\caption{\ourloop{} with two proposals}
\label{alg:loop}
\SetKwInOut{Input}{Input}
\SetKwInOut{Global}{Global}
\SetKwFunction{Verify}{Verify}
\SetKwFunction{Spawn}{Spawn}
\SetKwProg{Fn}{Function}{:}{}
\SetKw{KwOutput}{output}
\Input{prompt $x_{1:n}$; pre-layers $\mathcal{P}$;
       recurrent-layers $\mathcal{B}$;
       post-layers $\mathcal{C}$;
       total recurrent depth $R$;
       proposal depths $d_1<d_2<R$}
\Global{active branch set $\mathcal A\gets\varnothing$;
        output string $Y\gets\varepsilon$}
A branch $b$ has its prefix $X^{(b)}$, current hidden state
$H_b$, recurrent depth $r_b$, and proposal records $\mathcal Q_b$\;
\BlankLine
\Fn{\Spawn{$X$}}{
  $c\gets\operatorname{NewBranch}()$\;
  $X^{(c)}\gets X$; \quad $H_c\gets\mathcal{P}(X^{(c)})$\;
  $r_c\gets 0$; \quad $\mathcal Q_c\gets\varnothing$\;
  $\mathcal A\gets\mathcal A\cup\{c\}$\;
  \KwRet{$c$}\;
}
\BlankLine
\Fn{\Verify{$b,p$}}{
  $\rho\gets p$\;
  \For{$j\gets1$ \KwTo $|\mathcal Q_b|$}{
    $(\widetilde{x}_{b}^{(j)},q_j,c_j)\gets\mathcal Q_b[j]$\;
    $U_j\sim\mathcal{U}(0,1)$\;
    \If{$U_j<\min\{1,\rho(\widetilde{x}_{b}^{(j)})/q_j(\widetilde{x}_{b}^{(j)})\}$}{
      \KwRet{$(\widetilde{x}_{b}^{(j)},c_j)$}\tcp*{accept}
    }
    $\rho\gets\rho\ominus q_j$\tcp*{reject}
  }
  \BlankLine
  $x\sim\rho$\;
  \KwRet{$(x,\operatorname{Spawn}(X^{(b)}\mathbin{\Vert}x))$}
}
\BlankLine
$H\gets\mathcal{P}(x_{1:n})$\;
\For{$r\gets1$ \KwTo $R$}{
  $H\gets\mathcal{B}(H)$\;
}
$x\sim\mathcal{C}(H)$\;
$Y\gets Y\mathbin{\Vert}x$\tcp*{commit first token after prefill}
\If{generation has not stopped}{
  $b_{\mathrm{init}}\gets\operatorname{Spawn}
    (x_{1:n}\mathbin{\Vert}x)$\;
}
\BlankLine
\While{generation has not stopped}{
  $\mathcal A_{\mathrm{cur}}\gets\mathcal A$\;
  \ForEach{$b\in\mathcal A_{\mathrm{cur}}$ \textbf{in parallel}}{
    $H_b\gets\mathcal{B}(H_b)$; \quad $r_b\gets r_b+1$\;
    \BlankLine
    \uIf{$r_b=d_1$ \tcp*[r]{first proposal}}{
      $q_1\gets\mathcal{C}(H_b)$; \quad $\widetilde{x}_{b}^{(1)}\sim q_1$\;
      $c_1\gets\operatorname{Spawn}(X^{(b)}\mathbin{\Vert}\widetilde{x}_{b}^{(1)})$\;
      $\mathcal Q_b[1]\gets(\widetilde{x}_{b}^{(1)},q_1,c_1)$\;
    }
    \ElseIf{$r_b=d_2$}{
      $(\widetilde{x}_{b}^{(1)},q_1,c_1)\gets\mathcal Q_b[1]$; \quad $\widetilde{q}_2\gets\mathcal{C}(H_b)$\;
      \If{$\widetilde{q}_2(\widetilde{x}_{b}^{(1)})<q_1(\widetilde{x}_{b}^{(1)})$ \tcp*[r]{gating mechanism}}{
        $q_2\gets\widetilde{q}_2\ominus q_1$; \quad $\widetilde{x}_{b}^{(2)}\sim q_2$ \tcp*[r]{residual second proposal}
        $c_2\gets\operatorname{Spawn}(X^{(b)}\mathbin{\Vert}\widetilde{x}_{b}^{(2)})$\;
        $\mathcal Q_b[2]\gets(\widetilde{x}_{b}^{(2)},q_2,c_2)$\;
      }
    }
  }
  \BlankLine
  $b_{\max}\gets\arg\max_{b\in\mathcal A}r_b$\;
  \If{$r_{b_{\max}}=R$}{
    $p\gets\mathcal{C}(H_{b_{\max}})$\;
    $(x^\star,c^\star)\gets\operatorname{Verify}(b_{\max},p)$\;
    $Y\gets Y\mathbin{\Vert}x^\star$\tcp*{commit a new decoded token}
    $\mathcal A\gets\operatorname{Reroot}(\mathcal A,c^\star)$\tcp*{keep $c^\star$ and its descendants}
  }
}
\KwOutput{$Y$}\;
\end{algorithm}

\end{document}